\documentclass[11pt,a4paper]{article}
\usepackage[hyperref]{acl2021}
\usepackage{times}
\usepackage{latexsym}

\usepackage{microtype}

\aclfinalcopy 

\usepackage{xspace,mfirstuc,tabulary}

\usepackage{amsmath}
\usepackage{amssymb}
\usepackage{amsthm}
\usepackage{amsfonts}
\usepackage{mathtools}
\usepackage{enumitem}
\usepackage{relsize}
\usepackage{booktabs}
\usepackage{tabularx}
\usepackage{multirow}
\usepackage{makecell}
\usepackage{textcomp}
\usepackage{ifmtarg}
\usepackage{stmaryrd}  
\usepackage{comment}
\usepackage{accents}

\usepackage{emoji}

\newcommand{\ucambridge}{\emoji[twitter]{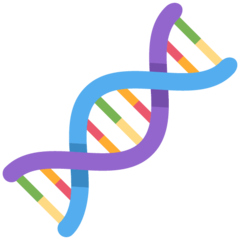}}
\newcommand{\ethz}{\emoji[twitter]{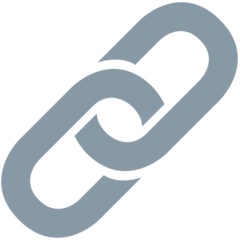}}
\newcommand{\jhu}{\emoji[twitter]{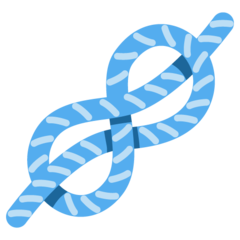}}

\author{
{Ran Zmigrod\raise1.0ex\hbox{\normalfont\ucambridge}\raise1.0ex\hbox{\normalfont}}~\;~Tim Vieira\raise1.0ex\hbox{\normalfont\jhu}~\;~Ryan Cotterell\raise1.0ex\hbox{\normalfont\ucambridge,\ethz}
\\
  \raise1.0ex\hbox{\normalfont\ucambridge}University of Cambridge~\;~\raise1.0ex\hbox{\normalfont\jhu}Johns Hopkins University~\;~\raise1.0ex\hbox{\normalfont\ethz}ETH Z\"{u}rich \\
  \texttt{rz279@cam.ac.uk}~\;~\texttt{tim.f.vieira@gmail.com} \\ \texttt{ryan.cotterell@inf.ethz.ch}
}

\date{}

\newtheorem{thm}{Theorem}
\newtheorem{cor}{Corollary}
\newtheorem{lemma}{Lemma}

\newtheorem{definition}{Definition}

\theoremstyle{definition}

\usepackage{cleveref}
\crefname{section}{\S}{\S\S}
\Crefname{section}{\S}{\S\S}
\crefname{table}{Tab.}{}
\crefname{figure}{Fig.}{}
\crefname{algorithm}{Alg}{}
\crefname{algorithm}{Alg}{}
\crefname{line}{Line}{}
\crefname{appendix}{App.}{}
\crefformat{section}{\S#2#1#3}
\crefname{thm}{Theorem}{}
\crefname{prop}{Proposition}{}
\crefname{defin}{Definition}{}
\crefname{lemma}{Lemma}{}
\crefname{cor}{Corollary}{}
\crefname{equation}{}{}

\usepackage{algorithm}
\usepackage{algorithmicx}
\usepackage[noend]{algpseudocode}
\algnewcommand{\parState}[1]{\State%
    \parbox[t]{\dimexpr\linewidth-\algmargin}{\strut\hangindent=\algorithmicindent \hangafter=1 #1\strut}}
\algrenewcommand\algorithmicindent{1.0em}%
\newcommand{\algorithmicdowhile}{\textbf{do}:}
\algdef{SE}[DOWHILE]{Do}{doWhile}{\algorithmicdowhile}[1]{\algorithmicwhile\ #1}%
\newcommand{\algorithmicfunc}[1]{\textbf{def} #1 :}
\algdef{SE}[FUNC]{Func}{EndFunc}[1]{\algorithmicfunc{#1}}{}

\newif\ifboldnumber

\algrenewcommand\alglinenumber[1]{%
  \footnotesize\ifboldnumber\color{red}\bfseries\fi\global\boldnumberfalse#1:}

\MakeRobust{\Call}   

\newcommand{\rightcomment}[1]{{\color{gray} \(\triangleright\) {\footnotesize\textit{#1}}}}
\algrenewcommand{\algorithmiccomment}[1]{\hfill \rightcomment{#1}}  
\algnewcommand{\LineComment}[1]{\State \rightcomment{#1}}
\algnewcommand{\LinesComment}[1]{\State \rightcomment{\parbox[t]{\linewidth-\leftmargin-\widthof{\(\triangleright\) }}{#1}}}

\renewcommand\algorithmicthen{:}

\makeatletter
\ifthenelse{\equal{\ALG@noend}{t}}%
  {\algtext*{EndFunc}}
  {}%
\makeatother

\algnewcommand{\IIf}[1]{\State\algorithmicif\ #1\ \algorithmicthen}
\algnewcommand{\EndIIf}{\unskip}

\makeatletter
\renewenvironment{proof}[1][\proofname]{\par
  \vspace{-0.5\topsep}
  \pushQED{\qed}%
  \normalfont
  \topsep0pt \partopsep0pt 
  \trivlist
  \item[\hskip\labelsep
        \itshape
    #1\@addpunct{.}]\ignorespaces
}{%
  \popQED\endtrivlist\@endpefalse
  \addvspace{5pt plus 5pt} 
}
\makeatother

\newcommand{\defn}[1]{\textbf{#1}}
\renewcommand{\th}[0]{^{\text{th}}}

\renewcommand{\bar}[1]{\overline{#1}}

\newcommand{\checkNotation}[1]{{#1}}

\newcommand{\prob}{\checkNotation{p}}

\newcommand{\nN}{\checkNotation{N}}

\newcommand{\real}{\checkNotation{\mathbb{R}}}
\newcommand{\realnn}{\checkNotation{\mathbb{R}_{\geq 0}}}

\newcommand{\plusequal}{{\,\textsf{+=}\,}}

\newcommand{\expect}[2]{\checkNotation{\mathbb{E}_{#1}\!\left[ {#2} \right]}}

\newcommand{\Z}{\checkNotation{\mathrm{Z}}}

\newcommand{\mat}[1]{\checkNotation{\mathbf{#1}}}
\newcommand{\matrixelem}[2]{\checkNotation{\mathrm{#1}_{#2}}}
\newcommand{\eye}{\mat{I}}
\newcommand{\inv}[1]{#1^{-1}}

\newcommand{\wfsa}{\checkNotation{\mathcal{M}}}
\newcommand{\vstart}{\checkNotation{\boldsymbol \alpha}}
\newcommand{\wstart}[1]{\checkNotation{\alpha_{#1}}}
\newcommand{\vend}{\checkNotation{\boldsymbol \omega}}
\newcommand{\wend}[1]{\checkNotation{\omega_{#1}}}

\newcommand{\alphabet}{\checkNotation{\mathcal{A}}}

\newcommand{\tuple}[1]{\checkNotation{{\langle #1 \rangle}}}

\newcommand{\runs}{\checkNotation{\mathcal{T}}}
\newcommand{\trajectory}[2]{\checkNotation{\tau_{#1 \rightsquigarrow #2}}}
\newcommand{\trajset}[2]{\checkNotation{\runs_{#1#2}}}

\newcommand{\wfsaW}{\checkNotation{\mat{W}}}

\newcommand{\wfsaWelempow}[2]{\checkNotation{\mathrm{W}_{#1}^{#2}}}

\newcommand{\transpose}[0]{^{\!\top}}

\newcommand{\kleene}[1]{#1^{\star}}
\newcommand{\wstar}{\kleene{\wfsaW}}
\newcommand{\wstarelem}[1]{\wfsaWelempow{#1}{\star}}

\newcommand{\perms}[1]{\checkNotation{\mathcal{S}_{#1}}}

\newcommand{\symwfsaW}[1]{\checkNotation{\wfsaW^{(#1)}}}
\newcommand{\symwfsaWelem}[2]{\checkNotation{\mathrm{W}^{(#1)}_{#2}}}

\newcommand*{\dt}[1]{%
  \accentset{\mbox{\large\bfseries .}}{#1}}
\newcommand{\hole}[3]{\checkNotation{{\dt{\mathrm{W}}^{(#3)}_{#1#2}}}}

\newcommand{\wfsatuple}{\tuple{\vstart, \{\symwfsaW{a} \}_{a \in \bar{\alphabet}}, \vend}}
\newcommand{\tran}[3]{\checkNotation{#1\xrightarrow{#3} #2}}

\newcommand{\edgetuple}{\vec{\tau}}

\newcommand{\funcSig}[3]{#1{:}\,\,#2 \mapsto #3}

\newcommand{\defeq}[0]{\overset{\smaller\mathrm{def}}{=}}
\newcommand{\stepeq}[1]{\overset{\smaller\mathrm{(#1)}}{=}}
\newcommand{\step}[1]{$\mathrm{(#1)}$}

\newcommand{\bigo}[1]{\mathcal{O}(#1)}
\newcommand{\abs}[1]{\lvert #1 \rvert}

\newcommand{\str}{\sigma}
\newcommand{\yield}[1]{\gamma\!\left( #1 \right)}
\newcommand{\alphsize}{\checkNotation{A}}

\newcommand{\timvCut}[1]{}
\newcommand{\saveforcameraready}[1]{\timv{#1}}
\renewcommand{\saveforcameraready}[1]{}  

\newcommand{\monstersum}[1]{{\sum_{\mathclap{\substack{#1}}}}}  
\newcommand{\bigmonstersum}[1]{{\mathlarger\sum_{\mathclap{\substack{#1}}}}}  

\newcommand{\monsterprod}[1]{{\prod_{\mathclap{\substack{#1}}}}}  
\newcommand{\monstersumtop}[2]{{\sum_{\mathclap{\substack{#1}}}^{#2}}}  
\newcommand{\monsterprodtop}[2]{{\prod_{\mathclap{\substack{#1}}}^{#2}}}  

\newcolumntype{C}{>{\centering\arraybackslash}X}

\newcommand{\algFace}[1]{\texttt{#1}}
\newcommand{\algCall}[2]{#1\!\left( #2 \right)}

\newcommand{\vstarStart}{\checkNotation{\mathbf{s}}}
\newcommand{\vstarEnd}{\checkNotation{\mathbf{e}}}
\newcommand{\wstarStart}[1]{\checkNotation{\matrixelem{s}{#1}}}
\newcommand{\wstarEnd}[1]{\checkNotation{\matrixelem{e}{#1}}}

\newcommand{\pair}[1]{\checkNotation{\tran{i_{#1}}{j_{#1}}{a_{#1}}}}
\newcommand{\pairp}[1]{\checkNotation{\tran{i'_{#1}}{j'_{#1}}{a'_{#1}}}}
\newcommand{\pairij}[0]{\checkNotation{\tran{i}{j}{a}}}

\newcommand{\algDm}[0]{\algFace{D}_m}

\newcommand{\algSecond}[0]{\algFace{E}_2}

\newcommand{\case}[1]{\noindent\emph{#1:}}

\allowdisplaybreaks

\title{Higher-order Derivatives of Weighted Finite-state Machines}

\begin{document}
\maketitle

\begin{abstract}
Weighted finite-state machines are a fundamental building block of NLP systems.
They have withstood the test of time---from their early use in noisy channel models in the 1990s up to modern-day neurally parameterized conditional random fields. 
This work examines the computation of higher-order derivatives with respect to the normalization constant for weighted finite-state machines. 
We provide a general algorithm for evaluating derivatives of all orders, which has not  been previously described in the literature. 
In the case of second-order derivatives, our scheme runs in the \emph{optimal} $\bigo{\alphsize^2\nN^4}$ time where $\alphsize$ is the alphabet size and $\nN$ is the number of states. Our algorithm is significantly faster than prior algorithms. 
Additionally, our approach leads to a significantly faster algorithm for computing second-order expectations, such as covariance matrices and gradients of first-order expectations.\looseness=-1
\end{abstract}

\section{Introduction}
Weighted finite-state machines (WFSMs) have a storied
role in NLP.
They are a useful formalism for speech recognition \cite{mohri-2002}, machine transliteration \citep{knight-98}, morphology \cite{geyken-2005, linden-2009} and phonology \citep{cotterell-etal-2015-modeling} \textit{inter alia}.
Indeed, WFSMs have been ``neuralized'' \citep{rastogi-etal-2016-weighting, hannun-iclr, schwartz-etal-2018-bridging}  and are still of practical use to the NLP modeler. 
Moreover, many popular sequence models, e.g., conditional random fields for part-of-speech tagging \citep{lafferty-2001}, are naturally viewed as special cases of WFSMs.
For this reason, we consider the study of algorithms for the WFSMs of interest \textit{in se} for the NLP community.

This paper considers inference algorithms for WSFMs.
When WFSMs are acyclic, there exist simple linear-time dynamic programs, e.g., the forward algorithm \citep{rabiner1989tutorial}, for inference.
However, in general, WFSMs may contain cycles and such approaches are not applicable.
Our work considers this general case and provides a method for efficient computation of $m^{\text{th}}$-order derivatives
over a cyclic WFSM.
To the best of our knowledge,
no algorithm for higher-order derivatives has been presented in the literature beyond a general-purpose method from automatic differentiation.
In contrast to many presentations of WFSMs \citep{mohri-1997}, 
our work provides a purely linear-algebraic take on them.
And, indeed, it is this connection that allows us to develop our general algorithm. 

We provide a thorough analysis of 
the soundness, runtime, and space complexity of our algorithm.
In the special case of second-order derivatives,
our algorithm runs \emph{optimally} in $\bigo{\alphsize^2\nN^4}$ time and space
where $\alphsize$ is the size of the alphabet,
and $\nN$ is the number of states.\footnote{Our implementation is available at \url{https://github.com/rycolab/wfsm}.}  
In contrast, the second-order expectation semiring of \citet{li-eisner-2009} provides an $\bigo{\alphsize^2\nN^7}$ solution 
and automatic differentiation \cite{griewank-1989} yields a slightly faster $\bigo{\alphsize\nN^5 \!+\! \alphsize^2\nN^4}$ solution.
Additionally, we provide a speed-up for the general family of second-order expectations.
Indeed, we believe our algorithm is the fastest known for computing common quantities, e.g., a covariance matrix.\footnote{Due to space constraints, we keep the discussion of our paper theoretical, though applications of expectations that we can compute are discussed in \citet{li-eisner-2009}, \citet{sanchezR20}, and \citet{zmigrod-2021-tacl}.}

\section{Weighted Finite-State Machines}
In this section we briefly provide important notation for WFSMs and a classic result that efficiently finds the normalization constant for the probability distribution of a WFSM.

\begin{definition}[]
A \defn{weighted finite-state machine} $\wfsa$ is a tuple $\wfsatuple$
where $\alphabet$ is an alphabet of size $\alphsize$, $\bar{\alphabet}\defeq\alphabet\cup\{\varepsilon\}$,
each $a\!\in\!\bar{\alphabet}$ has a symbol-specific transition matrix $\symwfsaW{a}\!\in\! \realnn^{\nN \times \nN}$ where $\nN$ is the number of states,
and $\vstart,\vend \in \realnn^\nN$ are column vectors of start and end weights, respectively.
We define the matrix $\wfsaW\!\defeq\!\sum_{a\in\bar{\alphabet}}\symwfsaW{a}$.
\end{definition} 

\begin{definition}[]
A \defn{trajectory} $\trajectory{i}{\ell}$ 
is an ordered sequence of transitions from state $i$ to state $\ell$.
Visually, we can represent a trajectory by
\begin{equation*}
    \trajectory{i}{\ell}\defeq \tran{i}{j}{a}\cdots \tran{k}{\ell}{a'}
\end{equation*}
The \defn{weight} of a trajectory is
\begin{equation}\label{eq:weight}
    w(\trajectory{i}{\ell})\defeq\wstart{i}\left(\ \ \ \ \ \ \monsterprod{(\tran{j}{k}{a})\in\trajectory{i}{\ell}}\ \symwfsaWelem{a}{jk}\right)\wend{\ell}
\end{equation}
We denote the (possibly infinite) set of trajectories from $i$ to $\ell$ by $\trajset{i}{\ell}$ and the set of all trajectories by $\runs\defeq\bigcup_{i,\ell \in [\nN]}\trajset{i}{\ell}$.\footnote{$\abs{\runs}$ is infinite if and only if $\wfsa$ is cyclic.}
Consequently, when we say $\trajectory{i}{\ell}\in\runs$, we make $i$ and $\ell$ implicit arguments to which $\trajset{i}{\ell}$ we are accessing.
\end{definition}

\noindent We define the \defn{probability} of a trajectory $\trajectory{i}{\ell} \!\in \!\runs$,
\begin{equation}
 p(\trajectory{i}{\ell}) 
 \defeq \frac{w(\trajectory{i}{\ell})}{\Z}
\end{equation}
where
\begin{equation}
 \Z \defeq \vstart\transpose \sum_{k=0}^{\infty}\wfsaW^{k}\, \vend\ 
\end{equation}\footnotetext{Another formulation for $\Z$ is $\sum_{\trajectory{i}{\ell}\in\runs}w(\trajectory{i}{\ell})$.}{Of} course, $p$ is only well-defined when $0\!<\!\Z\!<\!\infty$.\footnote{This requirement is equivalent to $\wfsaW$ having a spectral radius $<1$.}
Intuitively, $\vstart\transpose \wfsaW^{k}\, \vend$ 
is the total weight of all trajectories of length $k$.
Thus, $\Z$ is the total weight of all possible trajectories as it sums over the total weight for each possible trajectory length.
\begin{thm}[Corollary 4.2, \citet{lehmann-1977}]\label{thm:kleene}
\begin{equation}\label{eq:wstar}
    \wstar\defeq\sum_{k=0}^{\infty}\wfsaW^{k}=\left(\eye-\wfsaW\right)^{-1}
\end{equation}
\end{thm}

\noindent Thus, we can solve the infinite summation that defines $\wstar$ by matrix inversion in $\bigo{\nN^3}$ time.\footnote{This solution technique may be extended to closed semirings \citep{stephen1956kleene, lehmann-1977}.}
\begin{cor}\label{thm:z}

\begin{equation}
\Z=\vstart^{\top}\wstar\, \vend
\end{equation}
\end{cor}
\begin{proof}
Follows from \cref{eq:wstar} in \cref{thm:kleene}.
\end{proof}
\noindent By \cref{thm:z}, we can find $\Z$ in $\bigo{\nN^3+\alphsize\nN^2}$.\footnote{Throughout this paper, 
we assume
a dense weight matrix and
that matrix inversion is $\bigo{\nN^3}$ time.
We note, however, that when the weight matrix is sparse and structured,
faster matrix-inversion algorithms exist that 
exploit the strongly connected components decomposition of the graph \citep{mohri-2000}.
We are agnostic to the specific inversion algorithm, 
but for simplicity we assume the aforementioned running time.
}

\paragraph{Strings versus Trajectories.}
Importantly, WFSMs can be regarded as weighted finite-state acceptors (WFSAs) which accept strings as their input. 
Each trajectory $\trajectory{i}{\ell}\in\runs$ has a \defn{yield} $\yield{\trajectory{i}{\ell}}$ which is the concatenation of the alphabet symbols of the trajectory.
The yield of a trajectory ignores any $\varepsilon$ symbols, a discussion regarding the semantics of $\varepsilon$ is given in \citet{hopcroft-2001}. As we focus on distributions over trajectories, we do not need special considerations for $\varepsilon$ transitions.
We do not consider distributions over yields in this work as
such a distribution requires constructing a latent-variable model
\begin{equation}
    p(\str) = \frac{1}{\Z}\ \,\monstersum{\trajectory{i}{\ell} \in \runs, \\
     \yield{\trajectory{i}{\ell}}=\str}\,\, w(\trajectory{i}{\ell})
\end{equation}
where $\str \in \alphabet^*$ and $\yield{\trajectory{i}{\ell}}$ is the yield of the trajectory. 
While marginal likelihood can be found efficiently,\footnote{This is done by intersecting the WFSA with another WFSA that only accepts $\sigma$.} many quantities, such as the entropy of the distribution over yields, are intractable to compute \citep{cortes2008computation}.

\section{Computing the Hessian (and Beyond)}\label{sec:hess}
In this section, we explore algorithms for efficiently computing the Hessian matrix $\nabla^2\Z$. We briefly describe two inefficient algorithms, which are derived by forward-mode and reverse-mode automatic differentiation.  Next, we propose an efficient algorithm which is based on a key differential identity.

\subsection{An $\bigo{\alphsize^2\nN^7}$ Algorithm with Forward-Mode Automatic Differentiation}
One proposal for computing the Hessian comes from \citet{li-eisner-2009} who
introduce a method based on semirings for computing a general family of quantities known as second-order expectations (defined formally in \cref{sec:expect}).
When applied to the computation of the Hessian their method reduces precisely to forward-mode automatic differentiation \citep[AD;][Chap 3.1]{griewank-walther}.
This approach requires that we ``lift'' the computation of $\Z$ to operate over a richer numeric representation known as \emph{dual numbers} \citep{clifford1871preliminary, pearlmutterS07}.  Unfortunately, the second-order dual numbers that we require to compute the Hessian introduce an overhead of $\bigo{\alphsize^2\nN^4}$ per numeric operation of the $\bigo{\nN^3}$ algorithm that computes $\Z$, which results in $\bigo{\alphsize^2\nN^7}$ time.

\subsection{An $\bigo{\alphsize\nN^5\!+\!\alphsize^2\nN^4}$ Algorithm with Reverse-Mode Automatic Differentiation}\label{sec:ad}
Another method for materializing the Hessian $\nabla^2\Z$ is through reverse-mode automatic differentiation (AD).
Recall that we can compute $\Z$ in $\bigo{\nN^3\!+\!\alphsize\nN^2}$,
and can consequently find $\nabla\Z$ in $\bigo{\nN^3\!+\!\alphsize\nN^2}$ 
using one pass of reverse-mode AD \citep[Chapter~3.3]{griewank-walther}.
We can repeat differentiation to materialize $\nabla^2\Z$.  Specifically, we run reverse-mode AD once for each element $i$ of $\nabla\Z$.  Taking the gradient of $(\nabla\Z)_i$ gives a row of the Hessian matrix, $\nabla[ (\nabla\Z)_i ] = [\nabla^2\Z]_{(i,:)}$. 
Since each of these passes takes time $\bigo{\nN^3\!+\!\alphsize\nN^2}$ (i.e., the same as the cost of $\Z$), and $\nabla\Z$ has size $\alphsize\nN^2$, 
the overall time is $\bigo{\alphsize\nN^5\!+\!\alphsize^2\nN^4}$.

\subsection{Our Optimal $\bigo{\alphsize^2\nN^4}$ Algorithm}
In this section, we will provide an $\bigo{\alphsize^2\nN^4}$-time and space algorithm for computing the Hessian.
Since the Hessian has size $\bigo{\alphsize^2\nN^4}$, 
no algorithm can run faster than this bound;
thus, our algorithm's time and space complexities are \emph{optimal}.
Our algorithm hinges on the following lemma, which shows that the each of partial derivatives of $\wstar$ can be cheaply computed given $\wstar$.

\begin{lemma}\label{lem:dwstar}
For $i,j,k,\ell\!\in\![N]$ and $a\!\in\!\bar{\alphabet}$
\begin{equation}
    \frac{\partial\wstarelem{i\ell}}{\partial\symwfsaWelem{a}{jk}}
    =
    \wstarelem{ij} \hole{j}{k}{a} \wstarelem{k\ell}
\end{equation}
where $\hole{j}{k}{a}$ is shorthand for $\partial \symwfsaWelem{a}{jk}$.
\end{lemma}
\begin{proof}
\begin{align*}
\frac{\partial\wstarelem{i\ell}}{\partial\symwfsaWelem{a}{jk}} 
&= \frac{\partial}{\partial\symwfsaWelem{a}{jk}}\left[\inv{(\eye-\wfsaW)}_{i\ell}\right] \\
&= - \wstarelem{ij} \frac{\partial}{\partial\symwfsaWelem{a}{jk}}\left[(\eye-\wfsaW)\right] \wstarelem{k\ell} \\
&= \wstarelem{ij} \hole{j}{k}{a} \wstarelem{k\ell}
\end{align*}
The second step uses Equation 40 of the Matrix Cookbook \citep{matrix-cookbook}.
\end{proof}

We now extend \cref{lem:dwstar} to express higher-order derivatives in terms of $\wstar$.
Note that as in \cref{lem:dwstar}, we will use $\hole{i}{j}{a}$ as a shorthand for the partial derivative $\partial \symwfsaWelem{a}{ij}$.

\begin{thm}\label{thm:higher}
For $m\!\ge\!1$ and $m$-tuple of transitions $\edgetuple=\tuple{\pair{1},\dots,\pair{m}}$
\begin{align}\label{eq:order-m}
    &\frac{\partial^{m}\Z}{\partial\symwfsaWelem{a_1}{i_1j_1}\cdots\partial\symwfsaWelem{a_m}{i_mj_m}} = \bigmonstersum{\vspace{4pt}\\\big\langle\pairp{1},\cdots,\pairp{m}\big\rangle \in \perms{\edgetuple}} \\
    & 
    \wstarStart{i'_1}
    \hole{i'_1}{j'_1}{a'_1}
    \wstarelem{j'_1 i'_2}
    \hole{i'_2}{j'_2}{a'_2} 
    \cdots \wstarelem{j'_{m-1} i'_m}
    \hole{i'_m}{j'_m}{a'_m}
    \wstarEnd{j'_m}
    \nonumber
\end{align}
where $\vstarStart=\vstart\transpose\wstar$, 
$\vstarEnd=\wstar\vend$
and $\perms{\edgetuple}$ is the multi-set of permutations of $\edgetuple$.\footnote{As $\edgetuple$ may have duplicates, $\perms{\edgetuple}$
can also have duplicates and so
must be a multi-set.}
\end{thm}
\begin{proof}
See \cref{app:higher}
\end{proof}

\begin{cor}\label{corr:hess}
For $i,j,k,l\!\in\![N]$ and $a,b\!\in\!\bar{\alphabet}$
\begin{align}
    &\frac{\partial^2\Z}{\partial\symwfsaWelem{a}{ij}\partial\symwfsaWelem{b}{kl}} = \\
    &\quad\quad \wstarStart{i}\hole{i}{j}{a}\wstarelem{jk}\hole{k}{l}{b}\wstarEnd{l} + \wstarStart{k}\hole{k}{l}{b}\wstarelem{li}\hole{i}{j}{a}\wstarEnd{j} \nonumber
\end{align}
\end{cor}
\begin{proof}
Application of \cref{thm:higher} for the $m{=}2$ case.
\end{proof}

\cref{thm:higher} shows that, if we have already computed $\wstar$, 
each element of the $m\th$ derivative can be found in $\bigo{m\, m!}$ time: We must sum over $\bigo{m!}$ permutations, where each summand is the product of $\bigo{m}$ items.
Importantly, for the Hessian ($m=2$), we can find each element in $\bigo{1}$ using \cref{corr:hess}.
Algorithm $\algDm$ in \cref{fig:algs} provides pseudocode for materializing the tensor containing the $m\th$ derivatives of $\Z$.

\begin{thm}\label{thm:alg}
For $m\! \ge\! 1$, algorithm $\algDm$ computes $\nabla^m\Z$ in $\bigo{\nN^3\!+\!m\,m!\,\alphsize^{m}\nN^{2m}}$ time and $\bigo{\alphsize^{m}\nN^{2m}}$ space.
\end{thm}
\begin{proof}
Correctness of algorithm $\algDm$ follows from \cref{thm:higher}.
The runtime and space bounds follow by needing to compute and store each combination of transitions.
Each line of the algorithm is annotated with its running time.
\end{proof}

\begin{cor}
The Hessian $\nabla^2\Z$ can be materialized in $\bigo{\alphsize^2\nN^4}$ time and $\bigo{\alphsize^2\nN^4}$ space. Note that these bounds are optimal.
\end{cor}
\begin{proof}
Application of \cref{thm:alg} for the $m{=}2$ case.
\end{proof}

\setlength{\belowcaptionskip}{-15pt}
\begin{figure}[t!]
    \centering
    \input{figures/algs}
    \caption{Algorithms}
    \label{fig:algs}
\end{figure}

\section{Second-Order Expectations}\label{sec:expect}
In this section, we leverage the algorithms of the previous section to efficiently compute a family expectations, known as a second-order expectations.
To begin, we define an \defn{additively decomposable}
function $\funcSig{r}{\runs}{\real^R}$ as any function expressed as 
\begin{equation}
    r(\trajectory{i}{\ell})=\monstersum{(\tran{j}{k}{a})\in\trajectory{i}{\ell}}\ r_{jk}^{(a)}
\end{equation}
where each $r_{jk}^{(a)}$ is an $R$-dimensional vector.
Since many $r$ of interest are sparse,
we analyze our algorithms in terms of $R$ and its maximum density $R' \defeq \max_{\tran{j}{k}{a}} \| r_{jk}^{(a)} \|_0$.
Previous work has considered expectations of such functions \citep{eisner-2001} and the \emph{product} of two such functions \citep{li-eisner-2009}, better known as second-order expectations.
Formally, given two additively decomposable functions $\funcSig{r}{\runs}{\real^{R}}$ and $\funcSig{t}{\runs}{\real^{T}}$, a \defn{second-order expectation} is 
\begin{align}
    &\expect{\trajectory{i}{\ell}}{r(\trajectory{i}{\ell})t(\trajectory{i}{\ell})^\top} \defeq \\
    & \quad\quad\quad\quad\quad\quad \monstersum{\trajectory{i}{\ell}\in\runs}\ p(\trajectory{i}{\ell})r(\trajectory{i}{\ell})t(\trajectory{i}{\ell})^\top \nonumber
\end{align}
Examples of second-order expectations include the Fisher information matrix
and the gradients of first-order expectations 
(e.g., expected cost, entropy, and the Kullback--Leibler divergence).

Our algorithm is based on two fundamental concepts.
Firstly, expectations for probability distributions as described in \cref{eq:weight}, can be decomposed as expectations over transitions \citep{zmigrod-2021-tacl}.
Secondly, the marginal probabilities of transitions are connected to derivatives of $\Z$.\footnote{This is commonly used in the case of single transition marginals, which can be found by $\nabla\log\Z$}

\begin{lemma}\label{lemma:hole}
For $m\ge 1$ and $m$-tuple of transitions  $\edgetuple=\tuple{\pair{1},\dots,\pair{m}}$
\begin{align}
    \prob(\edgetuple) = \frac{1}{\Z}\monstersumtop{n=1}{m}\frac{\partial^n \Z}{\partial\symwfsaWelem{a_1}{i_1 j_1}\dots\partial\symwfsaWelem{a_n}{i_n j_n}}\monsterprodtop{k=1}{n}\symwfsaWelem{a_k}{i_k j_k}
\end{align}
\end{lemma}
\begin{proof}
See \cref{app:hole}.
\end{proof}

We formalize our algorithm as $\algSecond$ in \cref{fig:algs}. Note that we achieve an additional speed-up by exploiting associativity (see \cref{app:expect}).

\begin{thm}
Algorithm $\algSecond$ computes the second-order expectation of additively decomposable functions 
$\funcSig{r}{\runs}{\real^{R}}$ and 
$\funcSig{t}{\runs}{\real^{T}}$ in:
\begin{align*}
&\bigo{\nN^3 \!+\! \nN^2(\bar{R}\, \bar{T} \!+\! \alphsize R' T')} \text{ time} \\
&\bigo{\nN^2 \!+\! RT \!+\! \nN (R + T)} \text{ space}
\end{align*}
where $\bar{R}{=}\min(\nN R', R)$ and $\bar{T}{=}\min(\nN T', T)$.
\end{thm}
\begin{proof}
Correctness of algorithm $\algSecond$ is given in \cref{app:expect}.
The runtime bounds are annotated on each line of the algorithm.
We note that each $\widehat{r}$ and $\widehat{t}$ is $\bar{R}$ and $\bar{T}$ sparse.
$\bigo{\nN^2}$ space is required to store $\wstar$, $\bigo{RT}$ is required to store the expectation, and $\bigo{\nN (R + T)}$ space is required to store the various $\widehat{r}$ and $\widehat{t}$ quantities.
\end{proof}

Previous approaches for computing second-order expectations are significantly slower than $\algSecond$.
Specifically, using \citet{li-eisner-2009}'s second-order expectation semiring requires augmenting the arc weights to be $R \times T$ matrices and so runs in
$\bigo{\nN^3RT \!+\! \alphsize\nN^2 RT}$.
Alternatively, we can use AD, as in \cref{sec:ad}, to materialize the Hessian and compute the pairwise transition marginals.
This would result in a total runtime of 
$\bigo{\alphsize\nN^5 \!+\! \alphsize^2\nN^4R'T'}$.

\section{Conclusion}
We have presented efficient methods that exploit properties of the derivative of a matrix inverse to find $m$-order derivatives for WFSMs.
Additionally, we provided an explicit, novel, algorithm for materializing the Hessian in its \emph{optimal} complexity, $\bigo{\alphsize^2\nN^4}$.
We also showed how this could be utilized to efficiently compute second-order expectations of distributions under WFSMs, such as covariance matrices and the gradient of entropy.
We hope that our paper encourages future research to use the Hessian and second-order expectations of WFSM systems, which have previously been disadvantaged by inefficient algorithms.

\section*{Acknowledgments}
We would like to thank the reviewers for engagine with our work and providing valuable feedback. The first author is supported by the University of Cambridge School of Technology Vice-Chancellor's Scholarship as well as by the University of Cambridge Department of Computer Science and Technology's EPSRC.

\section*{Ethical Concerns}
We do not foresee how the more efficient algorithms presented this work exacerbate any existing ethical concerns with NLP systems.

\bibliography{acl2021}
\bibliographystyle{acl_natbib}

\clearpage
\appendix
\onecolumn
\section{Proofs}

\subsection{}\label{app:higher}
\setcounter{thm}{1}
\setcounter{equation}{5}
\begin{thm}
For $m\!\ge\!1$ and $m$-tuple of transitions $\edgetuple=\tuple{\pair{1},\dots,\pair{m}}$
\begin{align}
    &\frac{\partial^{m}\Z}{\partial\symwfsaWelem{a_1}{i_1j_1}\dots\partial\symwfsaWelem{a_m}{i_mj_m}} = \bigmonstersum{\vspace{4pt}\\\big\langle\pairp{1},\dots,\pairp{m}\big\rangle \in \perms{\edgetuple}} 
    \wstarStart{i'_1}
    \hole{i'_1}{j'_1}{a'_1}
    \wstarelem{j'_1 i'_2}
    \hole{i'_2}{j'_2}{a'_2} 
    \cdots \wstarelem{j'_{m-1} i'_m}
    \hole{i'_m}{j'_m}{a'_m}
    \wstarEnd{j'_m}
    \nonumber
\end{align}
where $\vstarStart=\vstart\transpose\wstar$, 
$\vstarEnd=\wstar\vend$
and $\perms{\edgetuple}$ is the multi-set of permutations of $\edgetuple$.
\end{thm}
\begin{proof}
We prove this by induction on $m$.

\case{Base Case} $m=1$ and $\edgetuple=\tuple{\pairij}$:
\begin{align*}
     \frac{\partial\Z}{\partial\symwfsaWelem{a}{ij}}
     = \frac{\partial}{\partial\symwfsaWelem{a}{ij}}\left[\ \monstersumtop{k,l=0}{\nN}\wstart{k}\wstarelem{kl}\wend{l}\right] 
     = \monstersumtop{k,l=0}{\nN}\wstart{k}\wstarelem{ki}\hole{i}{j}{a}\wstarelem{jl}\wend{l} 
     = \wstarStart{i}\hole{i}{j}{a}\wstarEnd{j}
\end{align*}

\newcommand{\aaa}[0]{{\color{blue}a}}
\newcommand{\ii}[0]{{\color{blue}i}}
\newcommand{\jj}[0]{{\color{blue}j}}
\newcommand{\maybetext}[1]{{[\color{red}\textbf{#1}]}\xspace}
\case{Inductive Step} Assume that the expression holds for $m$. Let $\edgetuple=\tuple{\pair{1},\dots,\pair{m}}$ and consider the tuple $\edgetuple\,'$, the concatenation of $(\tran{\ii}{\jj}{\aaa})$ and $\edgetuple$.
\begin{align*}
    &\frac{\partial^{m+1}\Z}{\symwfsaWelem{\aaa}{\ii \jj}\partial\symwfsaWelem{a_1}{i_1j_1}\dots\partial\symwfsaWelem{a_m}{i_mj_m}} 
    = \frac{\partial}{\partial\symwfsaWelem{\aaa}{\ii \jj}}\bigmonstersum{\vspace{1pt}\\\big\langle\pairp{1},\dots,\pairp{m}\big\rangle\in \perms{\edgetuple}} \wstarStart{i'_1}
    \hole{i'_1}{j'_1}{a'_1}
    \wstarelem{j'_1 i'_2}
    \cdots
    \hole{i'_m}{j'_m}{a'_m}
    \wstarEnd{j'_m}
\end{align*}
\noindent Consider the derivative of each summand with respect to $\symwfsaWelem{\aaa}{\ii \jj}$. By the product rule, we have
\begin{align*}
    &\frac{\partial}{\partial\symwfsaWelem{\aaa}{\ii \jj}}
    \left[
    \wstarStart{i'_1}
    \hole{i'_1}{j'_1}{a'_1}
    \wstarelem{j'_1 i'_2}
    \cdots
    \hole{i'_m}{j'_m}{a'_m}
    \wstarEnd{j'_m}
    \right] \\
    &= 
    \wstarStart{\ii}
    \hole{\ii}{\jj}{\aaa}
    \wstarelem{{\color{blue}j}i'_1}
    \hole{i'_1}{j'_1}{a'_1}
    \wstarelem{j'_1 i'_2}
    \cdots
    \hole{i'_m}{j'_m}{a'_m}
    \wstarEnd{j'_m} + \\
    & \quad \dots + 
    \wstarStart{i'_1}
    \cdots
    \wstarelem{j_{k} \ii}
    \hole{\ii}{\jj}{\aaa}
    \wstarelem{\jj i_{k+1}}
    \cdots
    \wstarEnd{j'_m} + \\
    &\quad \dots + \wstarStart{i'_1}
    \hole{i'_1}{j'_1}{a'_1}
    \wstarelem{j'_1 i'_2}
    \cdots
    \hole{i'_m}{j'_m}{a'_m}
    \wstarelem{j'_m \ii}
    \hole{\ii}{\jj}{\aaa}
    \wstarEnd{\jj}
\end{align*}
The above expression is equal to inserting $\tran{\ii}{\jj}{\aaa}$ in every spot of the induction hypothesis's permutation, thereby creating a permutation over $\edgetuple\,'$.
Reassembling with the expression for the derivative,
\begin{align*}
    &\frac{\partial^{m+1}\Z}{\partial\symwfsaWelem{\aaa}{\ii \jj}\partial\symwfsaWelem{a_1}{i_1j_1}\dots\partial\symwfsaWelem{a_m}{i_mj_m}} \!=\! \bigmonstersum{\vspace{1pt}\\\big\langle\pairp{1},\dots,\pairp{m+1}\big\rangle)\in \perms{\edgetuple\,'}}
    \wstarStart{i'_1}
    \hole{i'_1}{j'_1}{a'_1}
    \wstarelem{j'_1 i'_2}
    \hole{i'_2}{j'_2}{a'_2} 
    \cdots
    \hole{i'_{m+1}}{j'_{m+1}}{a'_{m+1}}
    \wstarEnd{j'_{m+1}}
    \nonumber
\end{align*}

\end{proof}

\subsection{}\label{app:hole}
\setcounter{lemma}{1}
\setcounter{equation}{9}
\begin{lemma}
For $m\ge 1$ and $m$-tuple of transitions  $\edgetuple=\tuple{\pair{1},\dots,\pair{m}}$
\begin{align}\label{eq:marginal}
    \prob(\edgetuple) = \frac{1}{\Z}\monstersumtop{n=1}{m}\ \frac{\partial^n \Z}{\partial\symwfsaWelem{a_1}{i_1 j_1}\dots\partial\symwfsaWelem{a_n}{i_n j_n}}\ \monsterprodtop{k=1}{n}\symwfsaWelem{a_k}{i_k j_k}
\end{align}
\end{lemma}
\begin{proof}
Let $\runs_{\edgetuple}$ be the set of trajectories such that $\trajectory{i}{\ell}\in\runs_{\edgetuple}\iff\edgetuple\subseteq\trajectory{i}{\ell}$.
Then,
\begin{align*}
    \prob(\edgetuple)=\frac{1}{\Z}\,\monstersum{\trajectory{i}{\ell}\in\runs_{\edgetuple}}w(\trajectory{i}{\ell})
\end{align*}
We prove the lemma by induction on $m$.

\case{Base Case}
Then, $m=1$ and $\edgetuple=\tuple{\pair{1}}$. We have that
\begin{align*}
    & \frac{1}{\Z}\frac{\partial \Z}{\partial\symwfsaWelem{a_1}{i_1 j_1}}\,\symwfsaWelem{a_1}{i_1 j_1} 
    = \frac{1}{\Z}\frac{\partial}{\partial\symwfsaWelem{a_1}{i_1 j_1}}\left[\sum_{\trajectory{i}{\ell}\in\runs}w(\trajectory{i}{\ell})\right]\symwfsaWelem{a_1}{i_1 j_1}
    \stepeq{a} \frac{1}{\Z}\left(\sum_{\trajectory{i}{\ell}\in\runs_{\edgetuple}}w(\trajectory{i}{\ell})\right) 
    = \prob(\pair{1})
\end{align*}
Step \step{a} holds because taking the derivative of $\Z$ with respect to $\symwfsaWelem{a_1}{i_1 j_1}$ yields the sum of the weights all trajectories which include $\pair{1}$ where we exclude $\symwfsaWelem{a_1}{i_1 j_1}$ from the computation of the weight.
Then, we can push the outer $\symwfsaWelem{a_1}{i_1 j_1}$ into the equation to obtain the sum of the weights of all trajectories containing $\pair{1}$.

\case{Inductive Step}
Suppose that \cref{eq:marginal} holds for any $m$-tuple.
Let $\edgetuple=\tuple{\pair{1},\dots,\pair{m+1}}$.
Without loss of generality, fix $\pair{1}$ and let $\edgetuple'$ be $\edgetuple$ without $\pair{1}$.
\begin{align*}
    & \frac{1}{\Z}\monstersumtop{n=1}{m+1}\frac{\partial^n \Z}{\partial\symwfsaWelem{a_1}{i_1 j_1}\dots\partial\symwfsaWelem{a_n}{i_n j_n}}\,\monsterprodtop{k=1}{n}\symwfsaWelem{a_k}{i_k j_k} \\
    &\stepeq{b} \symwfsaWelem{a_1}{i_1 j_1}\frac{\partial}{\partial\symwfsaWelem{a_1}{i_1 j_1}} \underbrace{\left[\frac{1}{\Z}\monstersumtop{n=2}{m+1}\frac{\partial^{(n-1)} \Z}{\partial\symwfsaWelem{a_2}{i_2 j_2}\dots\partial\symwfsaWelem{a_n}{i_n j_n}}\,\monsterprodtop{k=2}{n}\symwfsaWelem{a_k}{i_k j_k}\right]}_{\textrm{Inductive hypothesis}} \\
    &\stepeq{c} \symwfsaWelem{a_1}{i_1 j_1}\frac{\partial}{\partial\symwfsaWelem{a_1}{i_1 j_1}}\overbrace{\left[\frac{1}{\Z}\sum_{\trajectory{i}{\ell}\in\runs_{\edgetuple'}}w(\trajectory{i}{\ell})\right]} \\
    &\stepeq{d} \frac{1}{\Z}\frac{\partial}{\partial\symwfsaWelem{a_1}{i_1 j_1}}\left[ \sum_{\trajectory{i}{\ell}\in\runs_{\edgetuple'}}w(\trajectory{i}{\ell})\right] \symwfsaWelem{a_1}{i_1 j_1}\\
    &\stepeq{e} \prob(\edgetuple)
\end{align*}
Step \step{b} pushes $\frac{1}{\Z}$ and $\prod_{k=2}^{n}\symwfsaWelem{a_k}{i_k j_k}$ as constants into the derivative and step \step{c} uses our induction hypothesis on $\edgetuple'$.
Then, step \step{d} takes $\frac{1}{\Z}$ out of the derivative as we pushed it in as a constant.
Finally, step \step{e} follows by the same reasoning as step \step{a} in the base case above.
\end{proof}

\subsection{}\label{app:expect}
\setcounter{thm}{3}
\begin{thm}
Algorithm $\algSecond$ computes the second-order expectation of additively decomposable functions 
$\funcSig{r}{\runs}{\real^{R}}$ and 
$\funcSig{t}{\runs}{\real^{T}}$ in:
\begin{align*}
&\bigo{\nN^3 \!+\! \nN^2(\bar{R}\, \bar{T} \!+\! \alphsize R' T')} \text{ time} \\
&\bigo{\nN^2 \!+\! RT \!+\! \nN (R + T)} \text{ space}
\end{align*}
where $\bar{R}{=}\min(\nN R', R)$ and $\bar{T}{=}\min(\nN T', T)$.
\end{thm}
\begin{proof}
We provide a proof of correctness (the time and space bounds are discussed in the main paper).
\citet{zmigrod-2021-tacl} show that we can find second-order expectations over by finding the expectations over pairs of transitions. That is,
\begin{align*}
    &\expect{\trajectory{i}{\ell}}{r(\trajectory{i}{\ell})t(\trajectory{i}{\ell})\transpose} = 
    \monstersumtop{i,j,k,l=0}{\nN}\quad\  \monstersum{a,b\in\bar{\alphabet}}
    \ \prob\!\left(\tran{i}{j}{a}, \tran{k}{l}{b}\right)
    r_{ij}^{(a)}{t_{kl}^{(b)}}\transpose
\end{align*}
We can use \cref{lemma:hole} for the $m=2$ case, to find that the expectation is given by
\begin{align*}
     & \expect{\trajectory{i}{\ell}}{r(\trajectory{i}{\ell})t(\trajectory{i}{\ell})\transpose} \\
     &= \frac{1}{\Z}\bigg[\ \monstersumtop{i,j=0}{\nN}\ \ \monstersum{a\in\bar{\alphabet}}\ \frac{\partial\Z}{\partial\symwfsaWelem{a}{ij}}\symwfsaWelem{a}{ij}r_{ij}^{(a)}{t_{ij}^{(a)}}\transpose 
      + \monstersumtop{i,j,k,l=0}{\nN}\quad\ \monstersum{a,b\in\bar{\alphabet}}\ \frac{\partial^2\Z}{\partial\symwfsaWelem{a}{ij}\partial\symwfsaWelem{b}{kl}}\symwfsaWelem{a}{ij}\symwfsaWelem{b}{kl}r_{ij}^{(a)}{t_{kl}^{(b)}}\transpose\bigg]
\end{align*}
The first summand can be rewritten as
\begin{align*}
    &\monstersumtop{i,j=0}{\nN}\ \ \monstersum{a\in\bar{\alphabet}}\ \frac{\partial\Z}{\partial\symwfsaWelem{a}{ij}}\symwfsaWelem{a}{ij}r_{ij}^{(a)}{t_{ij}^{(a)}}\transpose 
    = \monstersumtop{i,j=0}{\nN}\ \ \monstersum{a\in\bar{\alphabet}}\wstarStart{i}\hole{i}{j}{a}\wstarEnd{j}\symwfsaWelem{a}{ij}r_{ij}^{(a)}{t_{ij}^{(a)}}\transpose
\end{align*}
The second summand can be rewritten as
\begin{align*}
    & \ \monstersumtop{i,j,k,l=0}{\nN}\quad\ \monstersum{a,b\in\bar{\alphabet}}\ \frac{\partial^2\Z}{\partial\symwfsaWelem{a}{ij}\partial\symwfsaWelem{b}{kl}}\symwfsaWelem{a}{ij}\symwfsaWelem{b}{kl}r_{ij}^{(a)}{t_{kl}^{(b)}}\transpose \\
    &\quad\quad = \monstersumtop{i,j,k,l=0}{\nN}\quad\ \monstersum{a,b\in\bar{\alphabet}} \wstarStart{i}\hole{i}{j}{a}\wstarelem{jk}\hole{k}{l}{b}\wstarEnd{l}\symwfsaWelem{a}{ij}\symwfsaWelem{b}{kl}r_{ij}^{(a)}{t_{kl}^{(b)}}\transpose + \wstarStart{k}\hole{k}{l}{b}\wstarelem{li}\hole{i}{j}{a}\wstarEnd{j}\symwfsaWelem{a}{ij}\symwfsaWelem{b}{kl}r_{ij}^{(a)}{t_{kl}^{(b)}}\transpose
\end{align*}
Consider the first summand of the above expression
\begin{align*}
    &\ \ \monstersumtop{i,j,k,l=0}{\nN}\quad\ \monstersum{a,b\in\bar{\alphabet}}\ \wstarStart{i}\hole{i}{j}{a}\wstarelem{jk}\hole{k}{l}{b}\wstarEnd{l}\symwfsaWelem{a}{ij}\symwfsaWelem{b}{kl}r_{ij}^{(a)}{t_{kl}^{(b)}}\transpose \\
    &{=}\monstersumtop{j,k=0}{\nN}\ \ \underbrace{\!\!\left[\monstersumtop{i=0}{\nN}\monstersum{a\in\bar{\alphabet}}\wstarStart{i}\hole{i}{j}{a}\symwfsaWelem{a}{ij}r_{ij}^{(a)}\right]\!\!}_{\defeq\widehat{r^s_j}}
    \wstarelem{jk} \underbrace{\!\left[\monstersumtop{l=0}{\nN}\monstersum{b\in\bar{\alphabet}}\hole{k}{l}{b}\wstarEnd{l}\symwfsaWelem{b}{kl}t_{kl}^{(b)}\right]\transpose\!\!}_{\defeq\widehat{t^e_k}\transpose} \\
    &{=}\monstersumtop{j,k=0}{\nN}\widehat{r^s_j}\wstarelem{jk}\widehat{t^e_k}\transpose
\end{align*}
Similarly, the second summand can be written as
\begin{align*}
    \monstersumtop{j,k=0}{\nN}\widehat{r^e_k}\wstarelem{jk}\widehat{t^s_j}\transpose
\end{align*}
Finally, recomposing all the pieces together,
\begin{align*}
    & \expect{\trajectory{i}{\ell}}{r(\trajectory{i}{\ell})t(\trajectory{i}{\ell})\transpose} 
    = \frac{1}{\Z}\bigg[\ \monstersumtop{i,j=0}{\nN}\ \ \widehat{r^s_i}\wstarelem{ij}\widehat{t^e_j}\transpose + \widehat{r^e_j}\wstarelem{ij}\widehat{t^s_i}\transpose + 
     \monstersum{a\in\bar{\alphabet}}\ \wstarStart{i}\hole{i}{j}{a}\wstarEnd{j}\symwfsaWelem{a}{ij}r_{ij}^{(a)}{t_{ij}^{(a)}}\transpose \bigg]\\
\end{align*}
\end{proof}

\end{document}